\documentclass[preprint]{article}

\usepackage{aistats2027}

\usepackage{amssymb}
\usepackage{amsfonts}
\usepackage{amsthm}
\usepackage{graphicx}
\usepackage[numbers,sort&compress]{natbib}

\usepackage[colorlinks=true, allcolors=blue]{hyperref}
\usepackage{xcolor}
\usepackage{algorithm}
\usepackage{algpseudocode}
\usepackage{booktabs}
\usepackage{multirow}
\usepackage{tikz}
\usetikzlibrary{arrows.meta,positioning}
\usepackage{array}

\newcommand{\E}{\mathbb{E}}
\newcommand{\N}{\mathcal{N}}

\newtheorem{proposition}{Proposition}
\newcommand{\numberthis}{\addtocounter{equation}{1}\tag{\theequation}}

\begin{document}

\twocolumn[

\aistatstitle{Spanning Diffusion and Flow-based Posterior Sampling
via Exact Gaussian Conditioning}

\aistatsauthor{Teo Deveney \And Lea Bogensperger \And Hok Shing Wong}

\aistatsaddress{University of Bath \And University of Zurich \And LMU Munich} ]

\begin{abstract}
\begin{sloppypar}
Flow-matching and diffusion models are increasingly used as priors for linear inverse problems, however, most methods are bespoke to one class of trajectory prior. Plug-and-play methods have been proposed, however these implicitly decorrelate the guided trajectory, resulting in broad priors that can struggle when data is sparse. This work proposes $\beta$-Flow: a general plug-and-play method indexed by a correlation parameter $\beta$ that generalises over flow, diffusion, and decorrelated priors, and a confidence parameter $\kappa$ that controls how strongly the path distribution is enforced. Each step requires only one forward evaluation of a denoising network, followed by a Gaussian conditioning step. This step is an exact Kalman update that is not dependent on network gradients. We show that many existing guided samplers can be interpreted as instances or approximations of our method, each implicitly using a different trajectory prior and operator in place of the Kalman gain. From this viewpoint, we provide a qualitative comparison of methods in terms of their implicit biases and assumptions. We demonstrate consistently state-of-the-art level performance across 7 inverse problems over 4 datasets in terms of PSNR and SSIM. Moreover, we demonstrate the utility of this prior control, by showing that confident high correlations (strict priors) work best when observations are less informative, while low correlation is preferred when observations are more informative, reflecting the complementary roles of the prior and likelihood in Bayesian inference. 

\end{sloppypar}
\end{abstract}

\section{Introduction}

\begin{figure*}[t]
\centering
\includegraphics[width=0.98\textwidth]{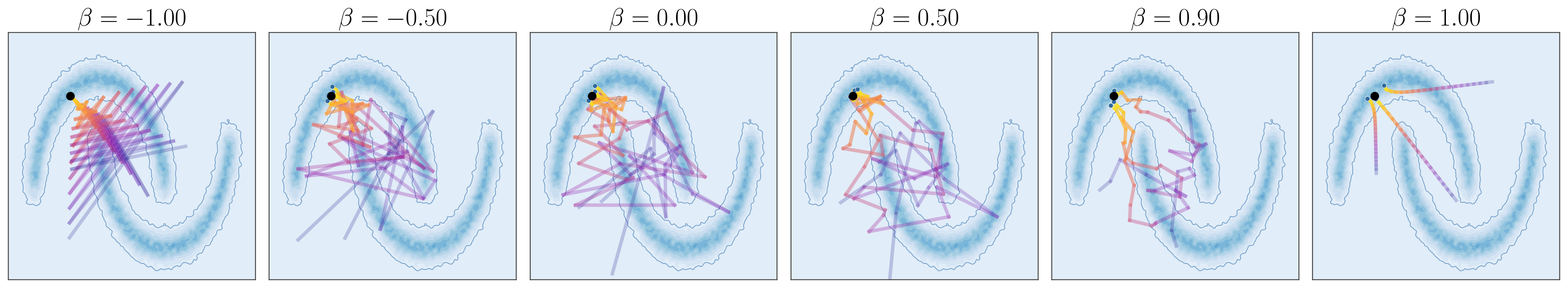}
\caption{Three guided trajectories on a toy denoising example for varying $\beta_t$. Each line is an independent trajectory, and the black marker is the observation. At $\beta_t=-1$ the trajectory is anti-correlated which results in the closest convergence to the observation, at $\beta_t=0$ every step re-noises independently, and as $\beta_t$ increases consecutive states become correlated, until at $\beta_t=1$ they follow a deterministic flow.}
\label{fig:beta-sweep}
\end{figure*}

Flow-matching \citep{lipman2023flow,liu2023rectified,albergo2023building} and diffusion models \citep{ho2020ddpm,song2021sde} generate data by transporting a simple noise distribution to the data distribution along a learned trajectory, deterministically or stochastically. This makes them natural priors for linear inverse problems. Given an observation $y$ of a clean sample $x_0$
\begin{align}
  \label{eq:observation}
  y = H(x_0) + \sigma_y\eta, \qquad \eta\sim\N(0,I),
\end{align}
the aim is to draw samples, or compute a point estimate from the posterior $p(x_0\mid y)$.

A large body of work approaches this problem by steering the model's reverse trajectory towards $y$ one step at a time. 
Most of these methods appear tailored to one model class, with likelihood score approximations common in diffusion models (DPS~\citep{chung2023dps}, $\Pi$GDM~\citep{song2023pigdm}), trajectory unrolling tied to flow-matching models (D-Flow~\citep{benhamu2024dflow}), and independent re-noising used in plug-and-play methods (PnP-Flow~\citep{martin2025pnpflow}, Flower~\citep{pourya2026flower}). Consequently, a scheme developed for one does not naturally carry over to the other without changing its prior behaviour. In this paper, we propose a general conditioning method for guided sampling that generalises over flow, diffusion and independent re-noising priors. 


\subsection{Preliminaries}
\label{sec:prelim}

Let $x_0\sim p_{\mathrm{data}}$ denote a data sample and $x_1\sim\N(0,I)$ independent noise. We consider the linear interpolant
\begin{align}
  \label{eq:interpolant}
  x_t = \alpha_t x_0 + \sigma_t x_1, \qquad \alpha_t = 1-t, \quad \sigma_t = t,
\end{align}
where $t\in [0,1]$ so that $t=0$ corresponds to the data, $t=1$ to pure noise.
Conditional on $x_0$, each $x_t$ is Gaussian with mean $\alpha_t x_0$ and covariance $\sigma_t^2 I$. A network trained on this interpolant, whether by flow-matching or as a diffusion denoiser, provides at every $x_t$ a one-step estimate $\hat x_0(x_t)\approx\E[x_0\mid x_t]$ of $x_0$.

Interpolant \eqref{eq:interpolant} defines a single deterministic path for each noise realisation $x_1$, but a guided sampler needs flexibility to move away from this path. We therefore model the current sampler state as a separate variable $\tilde x_{t_{i+1}}$
, whose correlation with the reference path is controlled by a time dependent parameter $\beta_t\in[-1,1]$,
\begin{align}
  \label{eq:augmented}
  \tilde x_{t_{i+1}} := \alpha_{t_{i+1}} x_0 + \sigma_{t_{i+1}}\big(\beta_{t_{i+1}} x_1 + \sqrt{1-\beta_{t_{i+1}}^2}\,\varepsilon\big),
\end{align}
where $\varepsilon \sim \N(0,I)$ is independent of $x_1$. This process has the same marginal as $x_{t_{i+1}}$ for every value of $\beta_{t_{i+1}}$.
For $\beta_{t_{i+1}}=1$ it fully coincides with $x_{t_{i+1}}$, for $\beta_{t_{i+1}}=0$ it is conditionally independent of it given $x_0$, and for $\beta_{t_{i+1}}<0$ they are negatively correlated. Together, the observation model \eqref{eq:observation} and interpolant \eqref{eq:interpolant} result in the graphical model of Figure~\ref{fig:dag}.

From here, it is natural to condition on everything that is known, i.e., $\tilde x_{t_{i+1}}$ and $y$. Each step the sampler must draw from $p(x_{t_i}\mid \tilde x_{t_{i+1}},y)$, which decomposes as
\begin{align*}
  \label{eq:step-target}
  &p(x_{t_i}\mid \tilde x_{t_{i+1}},y) \;=\; 
  \\ &\int \underbrace{p(x_0\mid \tilde x_{t_{i+1}},y)}_{\text{clean-signal posterior}}\;
  \underbrace{p(x_{t_i}\mid \tilde x_{t_{i+1}},x_0)}_{\text{transition kernel}} \,dx_0. \numberthis
\end{align*}
This decomposition admits a natural interpretation. The first factor is the posterior over the ground truth $x_0$ consistent with the data $y$ and current state $\tilde x_{t_{i+1}}$, and the second factor is the transition kernel, a stochastic interpolant \citep{albergo2023stochastic} between the endpoints $x_0$ and $\tilde x_{t_{i+1}}$ that transports this posterior along the prior trajectory. Every guided sampler approximates this integral, in practice by using a network estimate $\hat{x}_0 \approx \E[x_0\mid x_t]$ as a proxy for the unknown $x_0$.

\label{sec:coupled}
Standard Gaussian conditioning of \eqref{eq:augmented} (Appendix~\ref{app:derivation}, A.1) gives the transition kernel in closed form,
\begin{align}
  \label{eq:kernel}
  x_{t_i} \mid \tilde x_{t_{i+1}}, x_0 \;\sim\; \N\big(c\,x_0+\gamma\,\tilde x_{t_{i+1}},\ \tau^2I\big),
\end{align}
with coefficients $\gamma=\beta_{t_{i+1}}\sigma_{t_i}/\sigma_{t_{i+1}}$, $c=\alpha_{t_i}-\gamma\alpha_{t_{i+1}}$ and $\tau^2=\sigma_{t_i}^2(1-\beta_{t_{i+1}}^2)$. This gives the correct marginal of $x_{t_i}$ given $x_0$ for every $\beta_{t_{i+1}}$. Deterministic flow-matching updates and the standard diffusion posterior are both members of this family, differing only in their choice of $\beta_t$. Table~\ref{tab:unification} lists three choices of $\beta_t$ for which the kernel coincides with a known construction. We note that for $\beta_t\in[0,1]$ this derivation serves as an alternative construction of the DDIM family \citep{song2021ddim} (Appendix~\ref{app:derivation}, A.1). Negative values of $\beta_t$ lie outside the DDIM family but define an equally valid kernel.

\begin{table}[t]
\centering
\small
\setlength{\tabcolsep}{4pt}
\begin{tabular}{@{}l l >{\raggedright\arraybackslash}p{4.2cm}@{}}
\toprule
$\beta_{t_{i+1}}$ & $\tau^2$ & Coincides with \\
\midrule
$0$ & $\sigma_{t_i}^2$ &
  Independently noised $x_0$ \\[3pt]
$\beta^\star_{t_{i+1}}$ & $\sigma_{t_i}^2(1-\beta^{\star 2}_{t_{i+1}})$ &
  Reverse diffusion SDE \\[3pt]
$1$ & $0$ &
  Flow-matching ODE \\
\bottomrule
\end{tabular}

\caption{Common cases of the kernel \eqref{eq:kernel}. The diffusion schedule ($\beta^\star_{t_{i+1}} := \sigma_{t_i}\alpha_{t_{i+1}}/(\sigma_{t_{i+1}}\alpha_{t_i})$)
gives a Markovian forward process. Every schedule with $\beta_t\in[-1,1]$ defines a valid kernel with the correct marginals.}
\label{tab:unification}
\end{table}

\begin{figure}[H]
\centering
\begin{tikzpicture}[
  every node/.style={font=\small},
  state/.style={draw, circle, minimum size=9mm, inner sep=0pt},
  ghost/.style={draw, circle, dashed, minimum size=9mm, inner sep=0pt, gray, text=gray},
  obs/.style={draw=blue!55!black, circle, minimum size=9mm, inner sep=0pt, fill=blue!8, text=blue!55!black},
  >=Latex, thick
]
\node[align=center, font=\small\itshape] at (2.3,0.75) {true interpolant path};
\node[state] (x0) at (0,0)      {$x_0$};
\node[state] (xt) at (2.3,0)    {$x_{t_i}$};
\node[ghost] (xt1) at (4.6,0)   {$x_{t_{i+1}}$};
\draw[dotted, thick] (x0) -- (xt);
\draw[-{Latex[length=3mm]}, thick] (xt) -- (xt1) -- ++(1.4,0);

\node[state] (xtt) at (4.6,-1.9) {$\tilde x_{t_{i+1}}$};
\draw[->] (x0) -- (xtt);
\draw[dashed, gray] (xt1) -- (xtt) node[midway, right, black] {$\beta_{t_{i+1}}$};
\node[align=center, font=\small\itshape] at (4.6,-2.65) {current state};

\node[obs] (y) at (0,-1.9) {$y$};
\draw[->, blue!55!black] (x0) -- (y) node[midway, left, blue!55!black, font=\scriptsize] {$H,\sigma_y$};
\node[align=center, font=\small\itshape, text=blue!55!black] at (0,-2.65) {observation};
\end{tikzpicture}
\caption{Graphical model for one reverse step. The state $\tilde x_{t_{i+1}}$ held by the sampler depends on $x_0$ and is correlated, with coefficient $\beta_{t_{i+1}}$, with the reference point $x_{t_{i+1}}$ on the interpolant path (grey, dashed). It is conditionally independent of $x_{t_{i+1}}$ given $x_0$ when $\beta_{t_{i+1}}=0$ and coincides with it when $\beta_{t_{i+1}}=1$. Guided sampling proceeds from $x_1$ to $x_0$ by iteratively modelling the distribution of $x_{t_i}$ given $\tilde x_{t_{i+1}}$ and $y$.}
\label{fig:dag}
\end{figure}
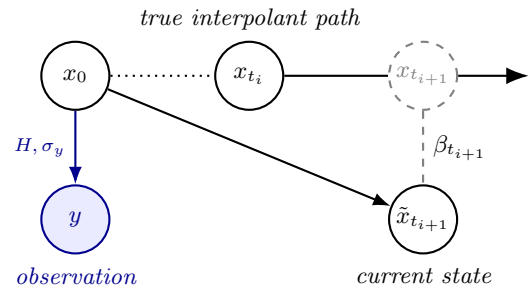
\subsection{Existing approaches}
\label{sec:existing}

Three families of guidance cover most of the literature. Each can be seen as approximating \eqref{eq:step-target}, and Section~\ref{sec:others} writes each in the update form derived in this paper.

The first estimates the ground truth $\hat x_0(x_{t_{i+1}})$ directly, applies a data consistency-based update, and samples $x_{t_i}$ independently of $x_{t_{i+1}}$. PnP-Flow \citep{martin2025pnpflow} does this via a gradient step as the data consistency update while FLOWER \citep{pourya2026flower} takes a proximal step. In both cases the corrected point is returned to the trajectory by re-noising with independent noise.

The second family, commonly used in diffusion settings, approximates the conditional score function by adding the score of an approximated likelihood to the unconditional score and integrating. This type of approximation is derived via the gradient of the likelihood evaluated at $\hat x_0(x_{t})$. DPS \citep{chung2023dps} and $\Pi$GDM \citep{song2023pigdm} are examples of this type of approach. The update is therefore based on the network's own local sensitivity to the state, and its step size is set by hand. 

The third family, most common in flow-matching settings, replaces the single step estimator with an estimate derived through a roll-out of the ODE solver. This is exemplified by D-Flow \citep{benhamu2024dflow}, which optimises the noise $x_1$ that seeds a deterministic sampler, backpropagating a data-fidelity loss on the final output $x_0$ through the entire discretised trajectory. Here the observation enters once, at the boundary, and the update is applied to $x_1$ rather than to a state on the trajectory. Closely related variants also apply this update at intermediate points of the trajectory~\cite{denker2026trajectory}. 


\subsection{Contributions}
\label{sec:contrib}

This paper treats a guided step as a Gaussian conditioning based on the graphical model of Figure \ref{fig:dag}. Our contributions are as follows.
\begin{itemize}
  \item \emph{One sampler for flow, diffusion and decorrelated priors.} A single correlation parameter $\beta_t$ spans the flow-matching ODE, the reverse diffusion SDE and independent re-noising, so one method serves all three model classes and the trajectory prior becomes a choice rather than implicitly inherited.
  \item \emph{Closed-form, gradient-free guidance.} Each guided step is an exact Gaussian conditioning update, computed as a Kalman update from a single forward evaluation of the network, with no backpropagation through the network or the trajectory.
  \item \emph{A comparison of existing methods within a unified view.} DPS, $\Pi$GDM, PnP-Flow, FLOWER and D-Flow can all be written as a choice of transport and gain, which makes their implicit assumptions explicit and comparable. We compare methods qualitatively based on the transport and gain they apply. For PnP-Flow and DPS we prove that no fixed step size in their gradient step can reproduce the calibrated Kalman gain, as both over-weight the observation increasingly as the trajectory approaches the data.
  \item \emph{Empirical results.} We show state-of-the-art performance against existing solvers, and demonstrate that the optimal choice of trajectory prior differs based on the problem, showing that control over the trajectory prior is practically beneficial.
\end{itemize}

\section{Methodology}
\label{sec:method}
Recall equation \eqref{eq:step-target}, which expresses the conditional distribution of the next state $x_{t_i}$ given the current state $\tilde x_{t_{i+1}}$ and the observation $y$. We note that this can be further decomposed, being proportional to
\begin{align}
  \label{eq:step-target-decomp}
  \int \underbrace{p(x_{t_i}\mid \tilde x_{t_{i+1}},x_0)}_{\text{transition kernel}}\;
 \underbrace{\underbrace{p(x_0\mid \tilde x_{t_{i+1}})}_{\text{prior}}\;
  \underbrace{p(y\mid x_0)}_{\text{likelihood}}}_{\text{clean-signal posterior}}
  \,dx_0.
\end{align}
Our approach is to carry out the Bayesian inference at the level of the clean signal and then transport the result along the trajectory. The observation model gives a likelihood, a prior over the clean signal combines with it to give the clean-signal posterior, and the transition kernel of Section~\ref{sec:coupled} transports that posterior to the next state, as in \eqref{eq:step-target}. Section~\ref{sec:update} derives the resulting update for a Gaussian prior with a given mean and covariance. It is available in closed form for every $\beta_t$, costs one forward evaluation of the network per step, and involves no network Jacobian. Section~\ref{sec:prior} then specifies our choice of mean and covariance using Tweedie's formulas.

\subsection{Guidance as closed form Gaussian conditioning}
\label{sec:update}

We assume for now that the prior over the clean signal, given the state the sampler holds, is Gaussian with a given mean $m_{\mathrm{pri}}$ and covariance $S_{\mathrm{pri}}$,
\begin{align}
  \label{eq:x0prior}
  x_0 \mid \tilde x_{t_{i+1}} \;\sim\; \N(m_{\mathrm{pri}},\ S_{\mathrm{pri}}).
\end{align}
These may depend on the current state $\tilde x_{t_{i+1}}$ and timestep $t+1$. We specify our choice of these in Section~\ref{sec:prior}. The observation model \eqref{eq:observation} gives the likelihood $p(y\mid x_0)=\N(Hx_0,\sigma_y^2I)$.

\emph{Posterior over the clean signal.} A Gaussian prior combined with a linear-Gaussian likelihood gives a Gaussian posterior, $p(x_0\mid\tilde x_{t_{i+1}},y)=\N(m_{\mathrm{post}},S_{\mathrm{post}})$, with
\begin{align*}
  \label{eq:x0post}
  m_{\mathrm{post}} &= m_{\mathrm{pri}} + S_{\mathrm{pri}}H^\top(HS_{\mathrm{pri}}H^\top+\sigma_y^2I)^{-1}r, \numberthis{} \\
  S_{\mathrm{post}} &= S_{\mathrm{pri}} - S_{\mathrm{pri}}H^\top(HS_{\mathrm{pri}}H^\top+\sigma_y^2I)^{-1}HS_{\mathrm{pri}}.
\end{align*}
The posterior mean is the prior mean plus a linear correction applied to the residual $r = y-Hm_{\mathrm{pri}}$ which incorporates the data. The linear operator applied to the residual is the Kalman gain, which trades off the prior covariance against the observation noise $\sigma_y^2$. A confident prior (small $S_{\mathrm{pri}}$) leaves $m_{\mathrm{pri}}$ almost unchanged, and a diffuse prior lets the observation dominate. The posterior covariance is the prior covariance reduced by what the observation resolves.

\emph{Transport along the trajectory.} By \eqref{eq:step-target} the next state is obtained by transporting this posterior through the kernel \eqref{eq:kernel}. The kernel is linear in $x_0$ and its noise is independent of $x_0$ and of $y$, so, given $\tilde x_{t_{i+1}}$ and $y$,
\begin{align}
  \label{eq:full-update}
  x_{t_i}|\ \tilde x_{t_{i+1}},y \sim \N\big(c\,m_{\mathrm{post}}+\gamma\,\tilde x_{t_{i+1}}, c^2S_{\mathrm{post}}+\tau^2\big),
\end{align}
where $(c,\gamma,\tau^2)$ are the kernel coefficients \eqref{eq:kernel}. That is, the posterior mean over $x_0$ is transported along the same path as the underlying kernel \eqref{eq:interpolant}, and the posterior variance expands the variance of \eqref{eq:interpolant} additively. These formulas can also be derived in a more opaque manner by direct Gaussian conditioning over the jointly Gaussian system $(x_0,x_t,\tilde x_{t_{i+1}},y)$. The mean in \eqref{eq:full-update} is therefore $\E[x_{t_i}\mid\tilde x_{t_{i+1}},y]$, the minimum-mean-squared-error estimate of $x_{t_i}$ under the prior.

In order to make a transition we compute the posterior transition mean $\mu_{\mathrm{post}} = c\,m_{\mathrm{post}}+\gamma\,\tilde x_{t_{i+1}}$ and add i.i.d. noise with standard deviation $\tau$ as prescribed by the kernel, giving an update rule \eqref{eq:kernel}
\begin{align}
  \label{eq:step}
  x_{t_i} = \mu_{\mathrm{post}} + \tau\,\xi, \qquad \xi\sim\mathcal N(0,I).
\end{align}
We have omitted the contribution of the $c^2S_{\mathrm{post}}$ covariance term. The reason for this is simple; the posterior covariance $S_{\mathrm{post}}$ over $x_0$ is a consequence of our choice of prior over $x_0$ \eqref{eq:x0prior} and the observation operator $H$ only. It does not necessarily represent the natural geometry of the data distribution and can introduce excess noise and unnatural structure if propagated. 

To reduce variance we draw $M$ independent copies of $\xi$, all of which are propagated. This mirrors the approach taken in \citep{pourya2026flower}.
Averaging the resulting outputs approximates marginalisation over the
residual uncertainty in $x_{t_i}$ at each step, instead of committing to a
single realisation at each step in the trajectory. Algorithm~\ref{alg:step}
summarises a step of our algorithm.

\begin{algorithm}[h]
\caption{One reverse step}
\begin{algorithmic}[1]
\Require known state $\tilde x_{t_{i+1}}$, observation $y$, operator $H$,
  noise level $\sigma_y$, network $f_\theta$, correlation $\beta_{t_{i+1}}$, scale $S_{t_{i+1}}$, particles $M$
\State $\hat x_0 \gets f_\theta(\tilde x_{t_{i+1}}, t_{i+1})$ 
\State compute $\gamma,c,\tau^2$ \eqref{eq:kernel}
\State $m_{\mathrm{post}} \gets$ posterior mean of $x_0$ \eqref{eq:x0post} 
\State $x_{t_i} \gets c\,m_{\mathrm{post}}+\gamma\,\tilde x_{t_{i+1}} + \tau\,\xi$ with $\xi\sim\mathcal N(0,I)$ 
\end{algorithmic}
\label{alg:step}
\end{algorithm}

Figure~\ref{fig:toy-beta-comparison} in Appendix~\ref{app:toy2d}
illustrates this step on a two-dimensional problem, showing the evolution of prior
and posterior over $x_{t_i}$ at six points of a trajectory for three
for different choices of $\beta_t$.

\subsection{The prior over $x_0$}
\label{sec:prior}

The only intractable term in \eqref{eq:step-target-decomp} is the prior $p(x_0\mid\tilde x_{t_{i+1}})$. Our method deals with this by using a Gaussian approximation, as is implicit in many diffusion guidance methods \citep{peng2024optimal}, which we motivate here using Tweedie's moment estimates. Given $\tilde x_{t_{i+1}}$ we can produce a one-step estimate $\hat x_0=\hat x_0(\tilde x_{t_{i+1}})=\mathbb{E}[x_0|\tilde x_{t_{i+1}}]$ (e.g.\ using Tweedie's formula~\citep{efron2011tweedie}). We use this as the mean of the Gaussian prior $m_{\mathrm{pri}}=\hat x_0$ \eqref{eq:x0prior}. We base the covariance $S_{\mathrm{pri}}=S_{t_{i+1}}$ on Tweedie's second order formula for the covariance of $x_0$ given $\tilde x_{t_{i+1}}$,
\begin{align}
  \label{eq:tweedie2}
  \frac{\sigma_{t_{i+1}}^2}{\alpha_{t_{i+1}}^2}
  \Big(I + \sigma_{t_{i+1}}^2\,\nabla^2\log p(\tilde x_{t_{i+1}})\Big).
\end{align}
In particular we take the data independent component to give a schedule for our prior covariance, and combine this with a tunable confidence parameter $\kappa$, giving
\begin{align}
  \label{eq:S}
  S_{t_{i+1}} = \kappa\left(\frac{\sigma_{t_{i+1}}}{\alpha_{t_{i+1}}}\right)^2 I.
\end{align}
The removal of the Hessian term is primarily a pragmatic one, since this is impractical to compute outright, and Jacobian-vector product-based implementations are significantly slower. We remark however, that this curvature is primarily negative in high likelihood regions of interest, so the omission of this term amounts to enlarging the variance in directions of high negative curvature rather than restricting it to no longer cover the required ground truth.

\subsection{Computational cost and stability}
\label{sec:cheap}
Every quantity in \eqref{eq:full-update} is an explicit function of
$\hat x_0$ and $\tilde x_{t_{i+1}}$, and $\hat x_0$ is obtained from a
single forward evaluation of the network that is held fixed for the
remainder of the step. Computing $x_{t_i}$ therefore requires no further
evaluation of $f_\theta$ and no differentiation through
it. The network enters only through $\hat x_0$, and its uncertainty
enters only through the scalar $S_{t_{i+1}}$ in the gain. This is the same cost profile as proximal \citep{martin2025pnpflow} and data-consistency samplers \citep{pourya2026flower}, whereas score-based guidance \citep{chung2023dps,song2023pigdm} obtains its update from the network's gradients, and flow-based rollout methods require gradients over the full trajectory rollout \citep{benhamu2024dflow}. A summary of network evaluations and gradients required per update between methods is summarised in Table~\ref{tab:gains}.

In terms of stability, our approach is constructed at the level of the joint Gaussian system, and therefore is not subject to discretisation error inherent to methods framed around integrating an ODE/SDE. Furthermore, since our method is not guided by network gradients, it remains robust in regions of low density where the lower quantity of training data can result in unpredictable gradients.

\section{Method Comparison}
\label{sec:others}

The update \eqref{eq:full-update} has three distinct parts. A prior over the clean
signal centred at a point estimate $\hat x_0$, a data informed update of that
estimate with the residual $r:=y-H\hat x_0$, and transport of the updated
estimate back along the trajectory. Every method considered here has this
structure, and can be written
\begin{align}
  \label{eq:template}
  x_{\mathrm{new}} \;=\; \mathcal T\big(\hat x_0+\mathcal G\,r;\ x_{\mathrm{cur}}\big),
\end{align}
so that each method is specified by its gain $\mathcal G$ and its
transport $\mathcal T$. For our update, FLOWER, PnP-Flow, DPS and
$\Pi$GDM, $x_{\mathrm{cur}}=\tilde x_{t_{i+1}}$, $x_{\mathrm{new}}=x_{t_i}$,
and $\mathcal T$ is the kernel \eqref{eq:kernel} for some $\beta_t$,
$\mathcal T(z;\tilde x_{t_{i+1}})=c\,z+\gamma\,\tilde x_{t_{i+1}}+\tau\,\xi$.
PnP-Flow and FLOWER have this structure by construction, and
Appendix~\ref{app:derivation}, A.2, shows that the score-based guidance
schemes DPS and $\Pi$GDM follow it implicitly. ODE rollout methods that optimise the latent noise $x_1$ have a slightly different structure, setting $x_{\mathrm{cur}}=x_1^{(k)}$, $x_{\mathrm{new}}=x_1^{(k+1)}$ (where $k$ indexes optimisation iterations). Here the point estimate is defined as the endpoint of the integration trajectory $\hat x_0=F(x_1^{(k)})$, and transport of the updated estimate follows by the reversal of this trajectory $\mathcal T(z;x_1^{(k)})=F^{-1}(z)$. For all methods we derive the implicit gains $\mathcal G$, which specify exactly how each method updates its estimate of $x_0$. Table~\ref{tab:gains} summarises the results, here we briefly summarise.

Our update is the case $\mathcal G=G(S_{t_{i+1}})$, where
$G(S):=SH^\top(HSH^\top+\sigma_y^2I)^{-1}$ is the Kalman gain of
\eqref{eq:x0post}. It automatically balances the prior's uncertainty against the observation noise, resulting in a properly normalised gain which stays bounded as the prior uncertainty $S$ grows. The other methods differ in how, and whether, they supply this normalisation.

\subsection{FLOWER and PnP-Flow}
\label{sec:flower}

These methods re-noise with independent noise in each iteration, and therefore implicitly fix $\beta_t=0$. FLOWER computes the posterior mean of $x_0$ exactly so $\mathcal G=G(S_{t_{i+1}})$
with $S_{t_{i+1}}=\sigma_{t_{i+1}}^2/(\alpha_{t_{i+1}}^2+\sigma_{t_{i+1}}^2)$.
FLOWER can therefore be considered as an instance of our method that uses a different prior approximation.
PnP-Flow instead takes a gradient step on the data fidelity,
$\mathcal G=\rho H^\top$, with $\rho$ tuned per problem. 

\subsection{DPS and $\Pi$GDM}
\label{sec:dps}

DPS adds to the unconditional diffusion step ($\beta_t=\beta^\star_t$) the
gradient of the misfit at the network estimate, differentiated through
the network. With $J:=\partial\hat x_0/\partial x_{t_{i+1}}$ and a
hand-set step size $\zeta$, its gain is
$\mathcal G_{\mathrm{DPS}}=\lambda_i JH^\top$ with $\lambda_i:=2\zeta/c$.
For an exact denoiser, Tweedie's formula gives
$J=(\alpha_{t_{i+1}}/\sigma_{t_{i+1}}^2)\operatorname{Cov}(x_0\mid x_{t_{i+1}})$,
so DPS approximates the Kalman gain with the prior covariance replaced by a
multiple of the network's own posterior covariance, and the inverse
residual covariance $(HSH^\top+\sigma_y^2I)^{-1}$ replaced by the
scalar $\lambda_i$.
$\Pi$GDM replaces the point mass by an isotropic Gaussian with
covariance $r_t^2I$, giving a gain proportional to
$J^\top H^\top(r_t^2HH^\top+\sigma_y^2I)^{-1}$. It contains the
normalisation that DPS lacks, with $r_t^2I$ in place of $S$, but still
uses the network Jacobian where the Kalman gain uses $S$.

\subsection{D-Flow}
\label{sec:dflow}

D-Flow optimises the noise $x_1$ that seeds a deterministic flow-matching model,
$x_0=F(x_1)$, by gradient descent on $\|y-HF(x_1)\|^2$. With
$\mathcal J:=\partial F/\partial x_1$ at the current iterate and step
size $\eta$, one step on $x_1$ changes the point estimate to
$\hat x_0+2\eta\,\bar{\mathcal J}\mathcal J^\top H^\top r$, where
$\bar{\mathcal J}$ is the average of $\partial F/\partial x_1$ along the
straight line from the current iterate to the next
(Appendix~\ref{app:derivation}, A.2). To first order
$\bar{\mathcal J}=\mathcal J$, and the gain has the form proportional to $\mathcal J\mathcal J^\top H^\top$, and so the update can be considered a gradient step preconditioned by $\mathcal J\mathcal J^\top$, where $\mathcal J$ is the Jacobian of the rollout $F$.

\subsection{Comparison summary}
\label{sec:comparison}

\begin{table*}[t]
\centering
\small
\renewcommand{\arraystretch}{1.2}
\begin{tabular}{@{}lcclc@{}}
\toprule
Method & Transport $\mathcal T$ & Point estimate & Gain $\mathcal G$ on the residual & Gradients through $f_\theta$ \\
\midrule
This paper & kernel, $\beta_t \in [-1,1]$ & one step & $G\big(\kappa(\sigma_{t_{i+1}}/\alpha_{t_{i+1}})^2\big)$ & none \\
FLOWER & kernel, $\beta_t=0$ & one step & $G\big(\sigma_{t_{i+1}}^2/(\alpha_{t_{i+1}}^2+\sigma_{t_{i+1}}^2)\big)$ & none \\
PnP-Flow & kernel, $\beta_t=0$ & one step & $\rho\,H^\top$ & none \\
DPS & kernel, $\beta^\star_t$ & one step & $\lambda_i\,JH^\top$ & every step \\
$\Pi$GDM & kernel, $\beta^\star_t$ & one step & $\propto J^\top H^\top(r_t^2HH^\top+\sigma_y^2I)^{-1}$ & every step \\
D-Flow & flow, $F^{-1}$ & integrated, $F(x_1^{(k)})$ & $2\eta\,\bar{\mathcal J}\mathcal J^\top H^\top$ & whole trajectory \\
\bottomrule
\end{tabular}
\caption{Every method's update written in the form of
\eqref{eq:template}, with $r=y-H\hat x_0$ and
$G(S)=SH^\top(HSH^\top+\sigma_y^2I)^{-1}$. For D-Flow, $\bar{\mathcal J}$ is the average of
$\mathcal J$ along the step. Appendix~\ref{app:derivation}, A.2, derives the DPS
and D-Flow entries.}
\label{tab:gains}
\end{table*}

Written in this form, the methods differ along three axes. The first is the transport; 
PnP-Flow and FLOWER use the kernel at $\beta_t=0$ and DPS at
$\beta_t=\beta^\star_t$, while D-Flow transports along the
learned flow. The flow and diffusion kernels are based on continuous ODE/SDE dynamics, whilst the PnP-Flow and FLOWER are discontinuous and considerably broader. The second is the
gain; our update and FLOWER apply a Kalman gain built from an explicit
prior on $x_0$, whereas PnP-Flow, DPS and D-Flow apply a gradient or
Jacobian step whose scalar step size stands in for the residual
covariance and so has to be tuned per problem. For PnP-Flow and DPS, Appendix~\ref{app:theory} shows that the Kalman gain cannot be reproduced by choice of step size: as $t\to0$ the Kalman gain vanishes while theirs does not, so both increasingly over-weight the observation (Proposition~\ref{prop:gain}). The third is the cost of obtaining the gain, with DPS and $\Pi$GDM requiring a gradient through the network at every step and D-Flow requiring the derivative of the whole trajectory, whereas a gain
built from the prior needs none.

\section{Numerical Results}
\label{sec:results}
\begin{table*}[t]
\centering
\caption{AFHQ-Cat performance. Most results were using 100 steps, with the exception of some pre-existing PnP-Flow/Flower baselines which used $500^*$ and $200^\dag$ steps in certain tasks. }
\label{tab:afhq}
\resizebox{\textwidth}{!}{%
\begin{tabular}{l ccc ccc ccc ccc ccc}
\toprule
& \multicolumn{3}{c}{Denoising}
& \multicolumn{3}{c}{Deblurring}
& \multicolumn{3}{c}{Super-resolution}
& \multicolumn{3}{c}{Random inpainting}
& \multicolumn{3}{c}{Box inpainting} \\
\cmidrule(lr){2-4}\cmidrule(lr){5-7}\cmidrule(lr){8-10}\cmidrule(lr){11-13}\cmidrule(lr){14-16}
Method
  & PSNR  & SSIM  & LPIPS
  & PSNR  & SSIM  & LPIPS
  & PSNR  & SSIM  & LPIPS
  & PSNR  & SSIM  & LPIPS
  & PSNR  & SSIM  & LPIPS \\
\midrule
Degraded
  & 20.00 & 0.314 & 0.509
  & 23.94 & 0.517 & 0.444
  & 11.70 & 0.208 & 0.873
  & 13.36 & 0.223 & 1.081
  & 21.80 & 0.740 & 0.198 \\
\midrule
PnP-GS
  & \textbf{32.58} & \textbf{0.894} & \textbf{0.072}
  & 27.91 & 0.753 & 0.349
  & 24.15 & 0.632 & 0.362
  & 29.42 & 0.836 & 0.126
  & -- & -- & -- \\
DiffPIR
  & 30.58 & 0.835 & 0.189
  & 27.56 & 0.728 & 0.342
  & 23.65 & 0.624 & 0.402
  & 31.70 & 0.881 & 0.062
  & -- & -- & -- \\
OT-ODE
  & 30.03 & 0.815 & \underline{0.076}
  & 27.06 & 0.713 & \textbf{0.123}
  & 25.91 & 0.716 & \textbf{0.108}
  & 29.40 & 0.839 & 0.090
  & 24.62 & 0.875 & 0.085 \\
D-Flow
  & 26.13 & 0.574 & 0.175
  & 27.82 & 0.721 & \underline{0.164}
  & 24.64 & 0.601 & 0.190
  & 32.20 & 0.894 & \textbf{0.040}
  & 26.26 & 0.842 & \underline{0.077} \\
Flow-Priors
  & 29.41 & 0.763 & 0.153
  & 26.47 & 0.700 & 0.181
  & 23.51 & 0.570 & 0.272
  & 32.37 & 0.906 & 0.047
  & 26.20 & 0.818 & 0.118 \\
PnP-Flow
  & 31.43 & 0.864 & 0.168
  & $28.19^*$ & $0.766^*$ & $0.332^*$
  & $\underline{27.37}^*$ & $\underline{0.774}^*$ & $\underline{0.183}^*$
  & $\underline{33.75}^\dag$ & $\textbf{0.929}^\dag$ & $0.048^\dag$
  & 26.68 & 0.901 & 0.120 \\
Flower
  & 32.35 & \underline{0.891} & 0.116
  & \underline{28.97} & \underline{0.784} & 0.283
  & $26.57^*$ & $0.750^*$ & $0.282^*$
  & $33.70^\dag$ & $0.927^\dag$ & $\underline{0.045}^\dag$
  & \textbf{26.88} & \textbf{0.922} & \textbf{0.066} \\
\textbf{$\boldsymbol{\beta}$-Flow (ours)}
  & \underline{32.55} & \textbf{0.894} & 0.089
  & \textbf{29.18} & \textbf{0.790} & 0.290
  & \textbf{28.38} & \textbf{0.797} & 0.186
  & \textbf{33.88} & \underline{0.928} & 0.057
  & \underline{26.86} & \underline{0.915} & 0.083 \\
\bottomrule
\end{tabular}}
\end{table*}

\begin{table*}[t]
\centering
\caption{Results on 100 test images of the dataset CelebA.}
\label{tab:celeba}
\resizebox{\textwidth}{!}{%
\begin{tabular}{l ccc ccc ccc ccc ccc}
\toprule
& \multicolumn{3}{c}{Denoising}
& \multicolumn{3}{c}{Deblurring}
& \multicolumn{3}{c}{Super-resolution}
& \multicolumn{3}{c}{Random inpainting}
& \multicolumn{3}{c}{Box inpainting} \\
\cmidrule(lr){2-4}\cmidrule(lr){5-7}\cmidrule(lr){8-10}\cmidrule(lr){11-13}\cmidrule(lr){14-16}
Method
  & PSNR  & SSIM  & LPIPS
  & PSNR  & SSIM  & LPIPS
  & PSNR  & SSIM  & LPIPS
  & PSNR  & SSIM  & LPIPS
  & PSNR  & SSIM  & LPIPS \\
\midrule
Degraded
  & 20.00 & 0.348 & 0.372
  & 27.83 & 0.740 & 0.126
  & 10.26 & 0.183 & 0.827
  & 11.95 & 0.196 & 1.041
  & 22.27 & 0.742 & 0.214 \\
\midrule
PnP-GS
  & \underline{32.64} & 0.910 & \underline{0.035}
  & 34.03 & 0.924 & 0.041
  & 31.31 & 0.892 & 0.064
  & 29.22 & 0.875 & 0.070
  & -- & -- & -- \\
DiffPIR
  & 31.20 & 0.885 & 0.060
  & 32.77 & 0.912 & 0.060
  & 31.52 & 0.895 & 0.033
  & 31.74 & 0.917 & 0.025
  & -- & -- & -- \\
OT-ODE
  & 30.54 & 0.859 & \textbf{0.032}
  & 33.01 & 0.921 & \textbf{0.029}
  & 31.46 & 0.907 & \textbf{0.025}
  & 28.68 & 0.871 & 0.051
  & 29.40 & 0.920 & 0.038 \\
D-Flow
  & 26.04 & 0.607 & 0.092
  & 31.25 & 0.854 & 0.038
  & 30.47 & 0.843 & \underline{0.026}
  & 33.67 & 0.943 & \textbf{0.015}
  & 30.70 & 0.899 & \underline{0.026} \\
Flow-Priors
  & 29.34 & 0.768 & 0.134
  & 31.54 & 0.858 & 0.056
  & 28.35 & 0.713 & 0.102
  & 32.88 & 0.871 & \underline{0.019}
  & 30.07 & 0.858 & 0.048 \\
PnP-Flow
  & 32.30 & \underline{0.911} & 0.056
  & 34.80 & 0.940 & 0.047
  & 31.49 & 0.906 & 0.056
  & \underline{33.98} & \underline{0.953} & 0.022
  & 31.09 & 0.940 & 0.043 \\
Flower
  & \textbf{33.14} & \textbf{0.926} & 0.038
  & \textbf{35.67} & \textbf{0.954} & \underline{0.032}
  & \underline{33.09} & \underline{0.932} & 0.040
  & 33.95 & \underline{0.953} & 0.020
  & \textbf{31.87} & \textbf{0.952} & \textbf{0.023} \\
\textbf{$\boldsymbol{\beta}$-Flow (ours)}
  & \textbf{33.14} & \textbf{0.926} & \underline{0.035}
  & \underline{35.63} & \textbf{0.954} & 0.036
  & \textbf{33.87} & \textbf{0.943} & 0.029
  & \textbf{34.52} & \textbf{0.959} & 0.022
  & \underline{31.58} & \underline{0.949} & \textbf{0.023} \\
\bottomrule
\end{tabular}}
\end{table*}

We evaluate the method across various inverse problems, including $4\times$ super-resolution, denoising, random inpainting, Gaussian deblurring, box inpainting, MRI, and CT reconstruction using a trained flow-matching models as the denoiser, and tested on a full held-out set of $N{=}100$ test images per problem. For AFHQ-Cats and CelebA we use the experimental setup and datasets, models, and comparison methods as \cite{pourya2026flower}, which is the current state-of-the-art on these problems. For the undersampled MRI and sparse-view CT reconstruction experiments we train a new flow-matching model on the FastMRI knee dataset \cite{Zbontar2018fastMRIAO,knoll2020fastMRI} and re-use a trained flow-matching model from~\cite{denker2026stability} trained on the AAPM dataset for CT~\cite{mccollough2017low} and compare all methods discussed in Section \ref{sec:others}. 

Results reported for our method are all based on a 100 step discretisation, and hyperparameters were tuned on a per-case basis for the highest PSNR using the left-out validation set. Comparison methods also all used 100 steps with the exception of two pre-existing baselines which used $200^*$ and $500^\dag$ steps on the AFHQ-Cats dataset consistent with \cite{pourya2026flower} (indicated by the corresponding superscripts in the table). The final reported numbers report the performance on the test dataset, which is independent of the train and validation sets.  We report averaged PSNR, SSIM~\cite{wang2004image} and LPIPS~\cite{zhang2018unreasonable} over these test samples, highlighting the best method in \textbf{bold} and the second best \underline{underlined}.

\subsection{AFHQ-Cats and CelebA}
\label{sec:results-afhq}
Following the set-up of \cite{martin2025pnpflow,pourya2026flower}, we evaluate on five inverse problems using the AFHQ-Cats and CelebA datasets. Results are shown in Tables \ref{tab:afhq} and \ref{tab:celeba}. For AFHQ-Cats our method has the best PSNR in 3 of 5 experiments, and is second ($\leq0.03$dB from the best) in the remaining two. A similar story holds for SSIM, where our method is highest in 3 of 5 experiments and second in the remaining ones by $<1\%$. For LPIPS our method is less performant, though very much competitive with the other approaches. This is not surprising, since our construction marginalises uncertainty to produce MMSE estimates, which (correctly) smooth over the unidentifiable fine structure that LPIPS is sensitive to. This is a classic demonstration of the perception-distortion tradeoff \cite{blau2018perception}; our method excels in accuracy, and therefore must sacrifice some perceptual capability. The results for CelebA are similar, with our method being best in 3 of 5 cases in PSNR and 4 of 5 in SSIM and second by a $<1\%$ margin in the remaining cases. Again LPIPS favours different approaches for state-of-the-art performance, but our method still remains competitive.

\begin{figure*}[htbp]
\centering
\includegraphics[width=0.89\textwidth]{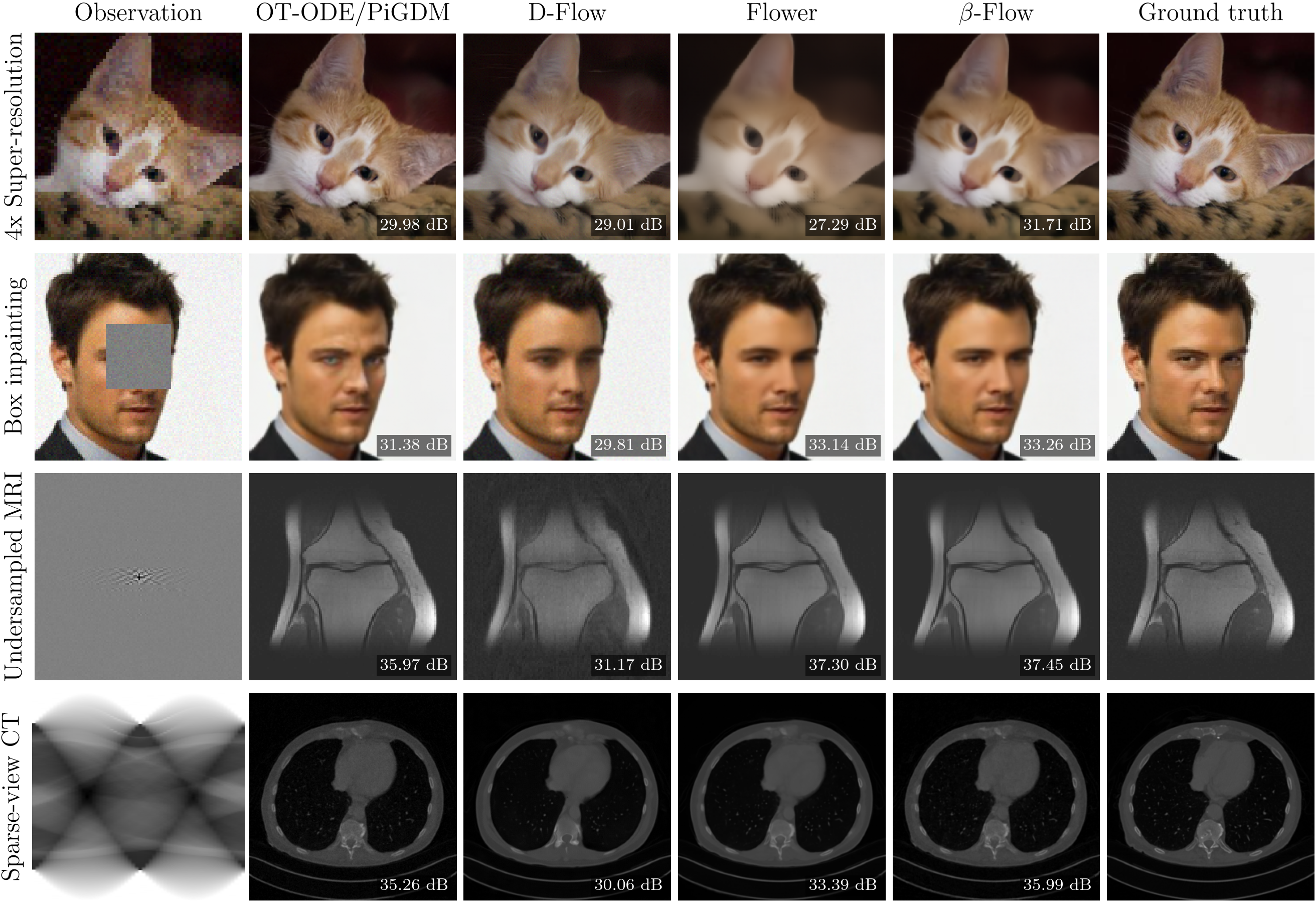}
\caption{Qualitative reconstructions across datasets and problems. All methods used 100 steps (100 updates in the case of D-Flow), and hyperparameters tuned for each dataset and task. Columns 2-4 represent reconstructions from each class of model. In Column 2 the likelihood score methods OT-ODE/$\Pi$GDM are applied to the upper/lower two rows respectively. These excel at producing realistic looking fine detail that does not necessarily match the ground truth. In Column 3 is the ODE rollout method D-Flow, which also produces some fine details, but appears sensitive to the degradation process as can be seen in the noise artifacts in row 2 and vertical smearing in row 3. Column 4 shows the independent-renoising method Flower, which achieves excellent accuracy when measurements are informative, but struggles at 100 steps with the sparser 4x super-resolution problem. In Column 5 $\beta$-Flow achieves excellent accuracy across all problems by adapting the prior to the task.}
\label{fig:all_problems_qualitative}
\end{figure*}

\subsection{Medical MRI/CT reconstruction}

We evaluate on medical image reconstruction problems, sparse-view CT and accelerated MRI, using for each problem a single pretrained general purpose flow-matching prior with no problem-specific retraining. For MRI, we use the FastMRI single-coil knee dataset  ($128\times128$, complex-valued, stored as 2 real channels).
The forward operator is an orthonormal 2D FFT followed by a variable-density random Cartesian mask retaining $20\%$ of phase-encode lines 
corresponding to an effective acceleration factor of $R=5$.
We evaluate on a held-out test split of 100 images.
For CT, we use the AAPM CT dataset ($256\times256$, single channel)~\cite{mccollough2017low}. The forward operator is a sparse-view Radon transform with 60 projection angles evenly spaced over $[0^\circ, 180^\circ)$, normalised to unit operator norm, with additive Gaussian noise of standard deviation $\sigma=0.01$ on the normalised sinogram. We evaluate on 50 test slices, evenly spaced across the AAPM test volume.

We compare our method, $\beta$-Flow, against five baselines using the same pretrained flow-matching prior. All method-specific hyperparameters were tuned on a held-out validation split, disjoint from the test images reported here. 
The quantitative results are shown in Table~\ref{tab:mri_ct_results}. $\beta$-Flow achieves state-of-the-art performance on both medical imaging tasks, demonstrating that our approach generalises well beyond natural images. 
On MRI reconstruction, Flower performs very closely to our method, with the two leading the comparison in PSNR, whereas on sparse-view CT, $\beta$-Flow yields clearly superior results in PSNR, which we attribute to the sparse-view CT problem being a more ill-conditioned problem. 
LPIPS performance is lower, consistent with our other experiments, which we attribute to the perception-distortion trade-off as discussed above~\cite{blau2018perception}. Our method also compares favourably in terms of runtime as shown in Table \ref{tab:mri_ct_runtime}. Fig.~\ref{fig:all_problems_qualitative} compares reconstructions for all datasets.

\begin{table}[t]
\centering
\footnotesize
\setlength{\tabcolsep}{2.2pt}
\caption{FastMRI and AAPM CT datasets results.}
\label{tab:mri_ct_results}
\begin{tabular}{l ccc ccc}
\toprule
& \multicolumn{3}{c}{Undersampled MRI}
& \multicolumn{3}{c}{Sparse-view CT} \\
\cmidrule(lr){2-4}\cmidrule(lr){5-7}
Method & PSNR & SSIM & LPIPS & PSNR & SSIM & LPIPS \\
\midrule
DPS                 & 36.19 & 0.926 & \textbf{0.037} & 28.79 & 0.740 & 0.107 \\
$\Pi$GDM            & 36.12 & 0.924 & \underline{0.039} & \underline{35.40} & 0.838 & 0.154 \\
D-Flow              & 34.19 & 0.871 & 0.158 & 27.42 & 0.740 & 0.157 \\
PnP-Flow            & 36.30 & 0.907 & 0.094 & 25.70 & 0.695 & 0.186 \\
Flower          & \underline{37.70} & \underline{0.942} & 0.091 & 33.41 & \textbf{0.905} & \textbf{0.076} \\
\textbf{$\boldsymbol{\beta}$-Flow (ours)} & \textbf{37.76} & \textbf{0.943} & 0.102 & \textbf{36.44} & \underline{0.898} & \underline{0.079} \\
\bottomrule
\end{tabular}
\end{table}


\subsection{$\beta_t$ and $\kappa$ super-resolution scale ablations}
Here we investigate the impact of prior hyperparameters $\kappa$ and $\beta$ , which both regularise the reconstruction in different ways. We apply super-resolution with the AFHQ-Cats dataset with increasingly severe downsampling scale and assess the impact of these prior hyperparameters on reconstruction quality. 

Intuitively, higher values of $\beta$ encourage trajectories to remain closer to flow paths, thus high $\beta$ values regularise by biasing reconstructions towards deterministic flows. This effect can be seen in Table \ref{tab:afhq_sr_beta} and Figure \ref{fig:afhq_sr_beta_qual}, where the optimal $\beta$ shows clear correlation with the super-resolution downsampling scale. Similarly, $\kappa$ regularises behaviour by controlling the relative importance of prior and observations at each step throughout the trajectory. A lower value of $\kappa$ implies a more confident prior, thus regularises behaviour by biasing it towards unconditional dynamics. This trend is clearly reflected in Table \ref{tab:afhq_sr_kappa} and is shown for images in Figure \ref{fig:afhq_sr_kappa_qual}, where the optimal $\kappa$ decreases monotonically as downsampling scale increases. Both effects are visualised on the 2d toy example in Figure \ref{fig:toy-beta-kappa}.

\begin{table}[H]
\centering
\small
\setlength{\tabcolsep}{1.9pt}
\begin{tabular}{@{}rcccccccc@{}}
\toprule
 & \multicolumn{2}{c}{$\times2$} & \multicolumn{2}{c}{$\times3$} & \multicolumn{2}{c}{$\times4$} & \multicolumn{2}{c}{$\times6$} \\
\cmidrule(lr){2-3}\cmidrule(lr){4-5}\cmidrule(lr){6-7}\cmidrule(lr){8-9}
$\beta_t$ & PSNR & SSIM & PSNR & SSIM & PSNR & SSIM & PSNR & SSIM \\
\midrule
$-0.6$ & $\mathbf{33.08}$ & $\mathbf{.909}$ & $29.87$ & $.838$ & $27.14$ & $.762$ & $22.00$ & $.609$ \\
$-0.3$ & $33.07$ & $.908$ & $29.93$ & $.839$ & $27.36$ & $.768$ & $22.40$ & $.620$ \\
$0.0$  & $33.04$ & $.907$ & $29.99$ & $\mathbf{.840}$ & $27.60$ & $.772$ & $22.84$ & $.632$ \\
$0.3$  & $32.96$ & $.904$ & $\mathbf{30.01}$ & $.839$ & $27.76$ & $.776$ & $23.40$ & $.645$ \\
$0.6$  & $32.76$ & $.896$ & $29.96$ & $.836$ & $\mathbf{27.85}$ & $\mathbf{.777}$ & $\mathbf{23.93}$ & $\mathbf{.657}$ \\
\bottomrule
\end{tabular}
\caption{PSNR (dB) and SSIM on AFHQ-Cat super-resolution across the
correlation $\beta_t$, with $\kappa=1$ fixed. The optimal $\beta_t$
increases with the downsampling factor, showing reconstructions favour more strongly correlated trajectories at higher scale factors.}
\label{tab:afhq_sr_beta}
\end{table}

\begin{table}[t]
\centering
\small
\setlength{\tabcolsep}{2.pt}
\begin{tabular}{@{}rcccccccc@{}}
\toprule
 & \multicolumn{2}{c}{$\times2$} & \multicolumn{2}{c}{$\times3$} & \multicolumn{2}{c}{$\times4$} & \multicolumn{2}{c}{$\times6$} \\
\cmidrule(lr){2-3}\cmidrule(lr){4-5}\cmidrule(lr){6-7}\cmidrule(lr){8-9}
$\kappa$ & PSNR & SSIM & PSNR & SSIM & PSNR & SSIM & PSNR & SSIM \\
\midrule
$0.1$  & $32.09$ & $.882$ & $30.15$ & $.845$ & $\mathbf{28.40}$ & $\mathbf{.798}$ & $\mathbf{25.57}$ & $\mathbf{.712}$ \\
$0.5$  & $33.00$ & $.907$ & $\mathbf{30.19}$ & $\mathbf{.846}$ & $28.22$ & $.790$ & $25.12$ & $.696$ \\
$1.0$  & $\mathbf{33.08}$ & $\mathbf{.909}$ & $30.01$ & $.839$ & $27.85$ & $.777$ & $23.93$ & $.657$ \\
$3.0$  & $32.99$ & $.904$ & $29.49$ & $.823$ & $26.47$ & $.735$ & $20.66$ & $.553$ \\
\bottomrule
\end{tabular}
\caption{PSNR (dB) and SSIM on AFHQ-Cat super-resolution for varying confidence scale $\kappa$. For each column $\beta_t$ is chosen to be the PSNR-optimal value from Table \ref{tab:afhq_sr_beta} (see Figure \ref{fig:afhq_sr_kappa_qual}). The optimal $\kappa$ decreases as the downsampling factor increases, favouring a more confident prior when observations are less informative.}
\label{tab:afhq_sr_kappa}
\end{table}

\section{Conclusion}

We have introduced a general plug-and-play technique for guided flow sampling for inverse problems that makes the prior behaviour itself an explicit choice. Our method allows an arbitrary denoising generative model to define trajectory priors ranging from deterministic flows to decorrelated, and even anti-correlated paths with controllable strength, while being simply interpreted as standard conditioning in a Gaussian graphical model. We illustrated the practical relevance of this control in super-resolution by demonstrating that harder tasks perform substantially better with straighter trajectories and stronger prior strength. Our method attains state-of-the-art performance in PSNR and SSIM in most tasks, and is $<1\%$ from the best performance in all tasks where it is not. Moreover, we demonstrated that many existing approaches can be interpreted as approximations to our approach, and presented a qualitative comparison across model classes that reveals the implicit assumptions contained in three common model classes.

\newpage



\subsubsection*{Acknowledgements}
This work was supported by the Maths4DL EPSRC Programme on the Mathematics of Deep Learning (EP/V026259/1), the UKRI Strength in Places Fund MyWorld Project (SIPF00006/1), and the Munich Center for Machine Learning (MCML).

\bibliographystyle{unsrtnat}
\bibliography{references}

\clearpage
\appendix
\thispagestyle{empty}

\onecolumn
\aistatstitle{Spanning Diffusion and Flow-based Posterior Sampling
via Exact Gaussian Conditioning: \\
Supplementary Materials}

\section{Derivations}
\label{app:derivation}

\paragraph{A.1 The kernel family.} Given $x_0$, the point
$x_{t_i}=\alpha_{t_i}x_0+\sigma_{t_i}x_1$ on the interpolant and the state
$\tilde x_{t_{i+1}}$ of \eqref{eq:augmented} are affine in the same noise
$x_1$, plus independent noise, so they are jointly Gaussian,
\[
  \begin{pmatrix} x_{t_i}\\ \tilde x_{t_{i+1}}\end{pmatrix}\Bigm|\, x_0
  \;\sim\; \N\!\left(
  \begin{pmatrix}\alpha_{t_i}x_0\\ \alpha_{t_{i+1}}x_0\end{pmatrix},\
  \begin{pmatrix}\sigma_{t_i}^2 I & \beta_{t_{i+1}}\sigma_{t_i}\sigma_{t_{i+1}} I\\
  \beta_{t_{i+1}}\sigma_{t_i}\sigma_{t_{i+1}} I & \sigma_{t_{i+1}}^2 I\end{pmatrix}
  \right),
\]
where the cross-covariance comes from the shared component $x_1$.
Further conditioning on $\tilde x_{t_{i+1}}$ gives
\begin{align}
  \label{eq:kernel-explicit}
  x_{t_i}\mid\tilde x_{t_{i+1}},x_0 \;\sim\;
  \N\Big(\alpha_{t_i}x_0+\beta_{t_{i+1}}\frac{\sigma_{t_i}}{\sigma_{t_{i+1}}}
  \big(\tilde x_{t_{i+1}}-\alpha_{t_{i+1}}x_0\big),\
  \sigma_{t_i}^2\big(1-\beta_{t_{i+1}}^2\big)I\Big),
\end{align}
which is \eqref{eq:kernel} with $\gamma=\beta_{t_{i+1}}\sigma_{t_i}/\sigma_{t_{i+1}}$,
$c=\alpha_{t_i}-\gamma\alpha_{t_{i+1}}$ and
$\tau^2=\sigma_{t_i}^2(1-\beta_{t_{i+1}}^2)$.

To recover the diffusion prior from this we note that DDPM is defined using the Markov forward step $x_{t_{i+1}}=a\,x_{t_i}+b\,\xi$, with
$\xi\sim\N(0,I)$ independent of $x_{t_i}$, $a=\alpha_{t_{i+1}}/\alpha_{t_i}$
and $b^2=\sigma_{t_{i+1}}^2-a^2\sigma_{t_i}^2$ so that the marginals are
preserved, has the reverse posterior
\[
  q(x_{t_i}\mid x_{t_{i+1}},x_0) = \N\Big(
  \frac{b^2\alpha_{t_i}}{\sigma_{t_{i+1}}^2}\,x_0
  +\frac{a\,\sigma_{t_i}^2}{\sigma_{t_{i+1}}^2}\,x_{t_{i+1}},\
  \frac{b^2\sigma_{t_i}^2}{\sigma_{t_{i+1}}^2}\,I\Big).
\]
Matching the coefficient of \eqref{eq:kernel-explicit}
gives $\beta_{t_{i+1}}=a\sigma_{t_i}/\sigma_{t_{i+1}}
=\sigma_{t_i}\alpha_{t_{i+1}}/(\sigma_{t_{i+1}}\alpha_{t_i})=\beta^\star_{t_{i+1}}$,
and at this value the coefficient of $x_0$ and the variance agree as well.

For the DDIM family we note that for marginals $\N(\alpha_tx_0,\sigma_t^2I)$, the DDIM
kernel with noise scale $s\in[0,\sigma_{t_i}]$ is
\[
  q_s(x_{t_i}\mid x_{t_{i+1}},x_0) = \N\Big(\alpha_{t_i}x_0
  +\sqrt{\sigma_{t_i}^2-s^2}\;\frac{x_{t_{i+1}}-\alpha_{t_{i+1}}x_0}{\sigma_{t_{i+1}}},\
  s^2I\Big).
\]
This is \eqref{eq:kernel-explicit} with
$\beta_{t_{i+1}}=\sqrt{1-s^2/\sigma_{t_i}^2}$, so $\beta_{t_{i+1}}\in[0,1]$
covers the DDIM family, and $s=0$ ($\beta_{t_{i+1}}=1$) is the deterministic
update used in flow-matching. 

Negative $\beta_{t_{i+1}}$ is not covered by DDIM, but nonetheless defines a valid kernel which introduces transitions that are actively negatively correlated. 

\paragraph{A.2 DPS and D-Flow in the form of \eqref{eq:template}.}
A DPS step proceeds by taking an unconditional reverse diffusion step and adding a gradient step based on the data misfit to the estimated ground truth. Taking $\beta_t=\beta^\star_t$, this can be written as
\begin{align}
  \label{eq:dps}
  \begin{split}
  x_{t_i} &= \underbrace{c\,\hat x_0(x_{t_{i+1}}) + \gamma\, x_{t_{i+1}} + \tau\,\xi}_{\text{Unconditional step}}
  \;-\; \zeta\,\nabla_{x_{t_{i+1}}}\|y-H\hat x_0(x_{t_{i+1}})\|^2\\
  &= c\big(\hat x_0+\lambda_iJ^\top H^\top r\big)+\gamma\,x_{t_{i+1}}+\tau\,\xi,
  \end{split}
\end{align}
where the second line uses
$\nabla_{x_{t_{i+1}}}\|y-H\hat x_0\|^2=-2J^\top H^\top r$ with
$J:=\partial\hat x_0/\partial x_{t_{i+1}}$, and $\lambda_i:=2\zeta/c$. This
is \eqref{eq:template} with the kernel at $\beta^\star_t$ and the gain
$\lambda_iJ^\top H^\top$.
For an exact denoiser $J=(\alpha_t/\sigma_t^2)\operatorname{Cov}(x_0\mid x_t)$, though a trained network
satisfies this only approximately. $\Pi$GDM's isotropic covariance
$r_t^2I$ changes the residual weighting to
$(r_t^2HH^\top+\sigma_y^2I)^{-1}$ but keeps $J$ in place of $S$.

\emph{D-Flow.} With $\hat x_0=F(x_1^{(k)})$ and
$\mathcal J:=\partial F/\partial x_1$ at $x_1^{(k)}$, a gradient step on
$\|y-HF(x_1)\|^2$ with step size $\eta$, and the update it induces on the
point estimate, are
\begin{align}
  \label{eq:dflow}
  \begin{split}
  x_1^{(k+1)} &= x_1^{(k)}+2\eta\,\mathcal J^\top H^\top r,\\
  F\big(x_1^{(k+1)}\big) &= \hat x_0+2\eta\,\bar{\mathcal J}\mathcal J^\top H^\top r,
  \end{split}
\end{align}
where $\bar{\mathcal J}:=\int_0^1\mathcal J(x_1^{(k)}+s\,\Delta x_1)\,ds$ is
the average of $\mathcal J$ along the step
$\Delta x_1:=x_1^{(k+1)}-x_1^{(k)}$. The second line is exact: the
fundamental theorem of calculus along the straight line between the
iterates gives $F(x_1^{(k+1)})-\hat x_0=\bar{\mathcal J}\,\Delta x_1$, the
identity behind Integrated Gradients \citep{sundararajan2017axiomatic}.
Applying $F^{-1}$ gives \eqref{eq:template} with $\mathcal T=F^{-1}$ and
the gain of Table~\ref{tab:gains}. To first order $\bar{\mathcal J}=\mathcal J$, and
$\mathcal J\mathcal J^\top$ is the covariance that a unit Gaussian
perturbation of $x_1$ induces on $F(x_1)$. Beyond first order
$\bar{\mathcal J}$ depends on $x_1^{(k+1)}$, and so on $r$, so the D-Flow
gain is not a fixed operator.

\section{Step-size gains versus the Kalman gain}
\label{app:theory}

Every method in Table~\ref{tab:gains} corrects the point estimate by
$\mathcal G r$, where $r=y-H\hat x_0$ is the residual. The step changes the
residual to
\begin{align}
  r'=y-H(\hat x_0+\mathcal G r)=(I-H\mathcal G)\,r,
\end{align}
so $H\mathcal G$ is the part of the residual the step removes. The residual
can also be written as $r = \sigma_y\eta + H(x-\hat x_0)$, thus it is a combination of the prediction error (projected into observation space), which is signal, and observation noise. The Kalman step removes a proportion of the residual that balances the uncertainty in the prior and observations based on Gaussian conditioning. A step that removes more of the residual over-weights the observations, either by treating the data as more certain, or the prior as less certain, leading to overfitting to the noise.

\paragraph{Setting.} Let $H$ be any linear operator, let
$C_t:=\operatorname{Cov}(x_0\mid x_t)$ be the posterior covariance at the
current time $t=t_{i+1}$, and assume $\|C_t\|\to0$ as $t\to0$. Write
$G(C)=CH^\top(HCH^\top+\sigma_y^2I)^{-1}$ for the Kalman gain built from a
covariance $C$. Ours is $G(S_{t_{i+1}})$ with $S_{t_{i+1}}$ of \eqref{eq:S},
and $G_t:=G(C_t)$ is the Kalman gain with the true posterior covariance.
PnP-Flow applies $g_t=\rho H^\top$, and DPS applies
$g_t=\lambda_iJ_t^\top H^\top$ with $\lambda_i=2\zeta/c$, where
$J_t:=\partial\hat x_0/\partial x_t$ is the Jacobian of the network. The following proposition shows that without tuned step size schedules, PnP-Flow and DPS cannot reproduce the calibrated step of the Kalman gain.

\begin{proposition}
\label{prop:gain}
Let $r$ be a residual with $H^\top r\neq0$, and let $t\to0$.
\begin{enumerate}
  \item[(i)] \emph{Kalman gains, including ours.} $HG_t\to0$ and
  $HG(S_{t_{i+1}})\to0$, for every $\kappa>0$.
  \item[(ii)] \emph{PnP-Flow.} $Hg_t=\rho HH^\top$ does not depend on $t$, so
  the step removes the share $\rho HH^\top r$ of the residual at every
  step.
  \item[(iii)] \emph{DPS.} Assuming $J_t\to I$ as $t\to 0$ (the denoiser output converges to its input in the limit of no noise) $\|Hg_tr\|\ge2\zeta\|HH^\top r\|\,(1+o(1))$, so the residual removed does not vanish.
\end{enumerate}
PnP-Flow and DPS therefore overweight the observations by continuing to reduce the residual as $t\to0$, where the true conditional mean recognises models the remaining residual as noise. In particular, no fixed step size
reproduces the Kalman gain along the trajectory, and these methods require scheduling as a result.
\end{proposition}

\begin{proof}
\leavevmode
\begin{enumerate}
  \item[(i)] With $A:=\sigma_y^{-2}HCH^\top\succeq0$,
  $HG(C)=A(I+A)^{-1}$, so $\|HG(C)\|\le\|A\|\le\|H\|^2\|C\|/\sigma_y^2$, which
  tends to zero for $C=C_t$ and for $C=S_{t_{i+1}}=\kappa(\sigma_t^2/\alpha_t^2)I$.
  \item[(ii)] $HH^\top r\neq0$ since $H^\top r\neq0$.
  \item[(iii)] $c=\alpha_{t_i}-\gamma\alpha_{t_{i+1}}\le1$ for the diffusion
  kernel ($\gamma\ge0$), so $\lambda_i\ge2\zeta$, and
  $HJ_t^\top H^\top r\to HH^\top r$ when $J_t\to I$.\qedhere
\end{enumerate}
\end{proof}

\section{Two-dimensional illustration}
\label{app:toy2d}
In this appendix we provide some further 2-dimensional visualisations in order to further inspect the mechanisms of the approach, which is difficult to see directly for the high-dimensional image problems of Section~\ref{sec:results}. We apply our construction to a two-moons denoising problem ($H=I$), so that the priors and posterior at each step are two-dimensional and can be plotted directly. Figure~\ref{fig:toy-beta-kappa} compares the impact of the correlation and confidence parameters $\beta,\kappa$ on the trajectories generated, and Figure~\ref{fig:toy-beta-comparison} visualises the prior and posterior distributions that are modelled during a single transition at various stages of sampling. Both figures share the same trained flow-matching network, the same true signal $x_0$ and observation $y$.

\begin{figure}[H]
\centering
\includegraphics[width=\textwidth]{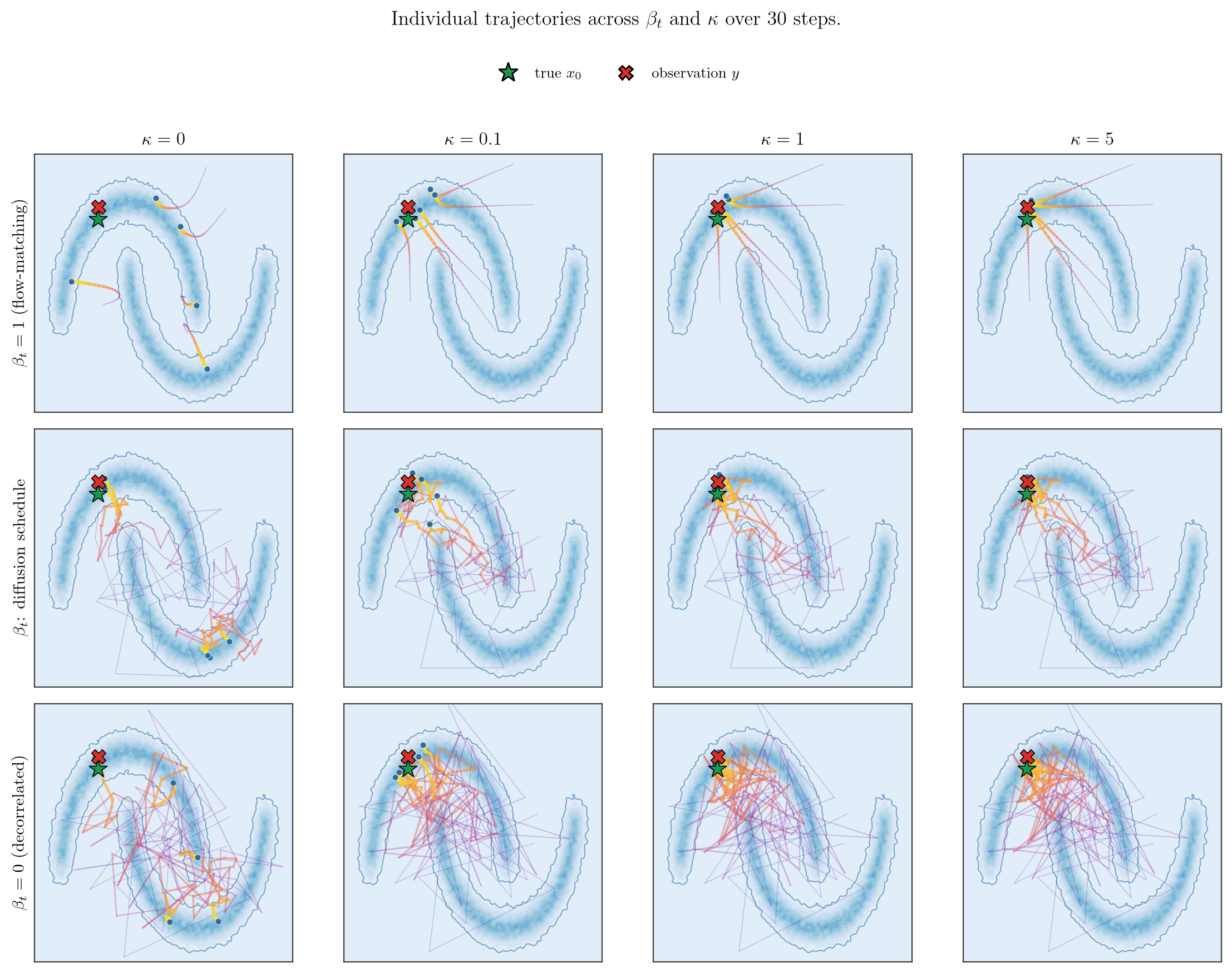}
\caption{Trajectories for a two-moons denoising problem under for trajectory priors in $\beta\in\{1,\beta_t^*,0\}$ and endpoint-uncertainty scale $\kappa \in \{0,0.1,1,5\}$. For $\kappa=0$ the Kalman gain vanishes, so the particles follow the unconditional prior unaffected by the observations. For $\kappa>0$ the prior weakens and lets the Kalman gain pulls the particles towards observation with the magnitude of the pull determined by $\kappa$. Trajectories in the top row are deterministic given the initial noise and observations, as expected of an ODE update, whereas those in the lower two rows take random transitions at each step.}
\label{fig:toy-beta-kappa}
\end{figure}

\begin{figure}[H]
\centering
\includegraphics[width=\textwidth]{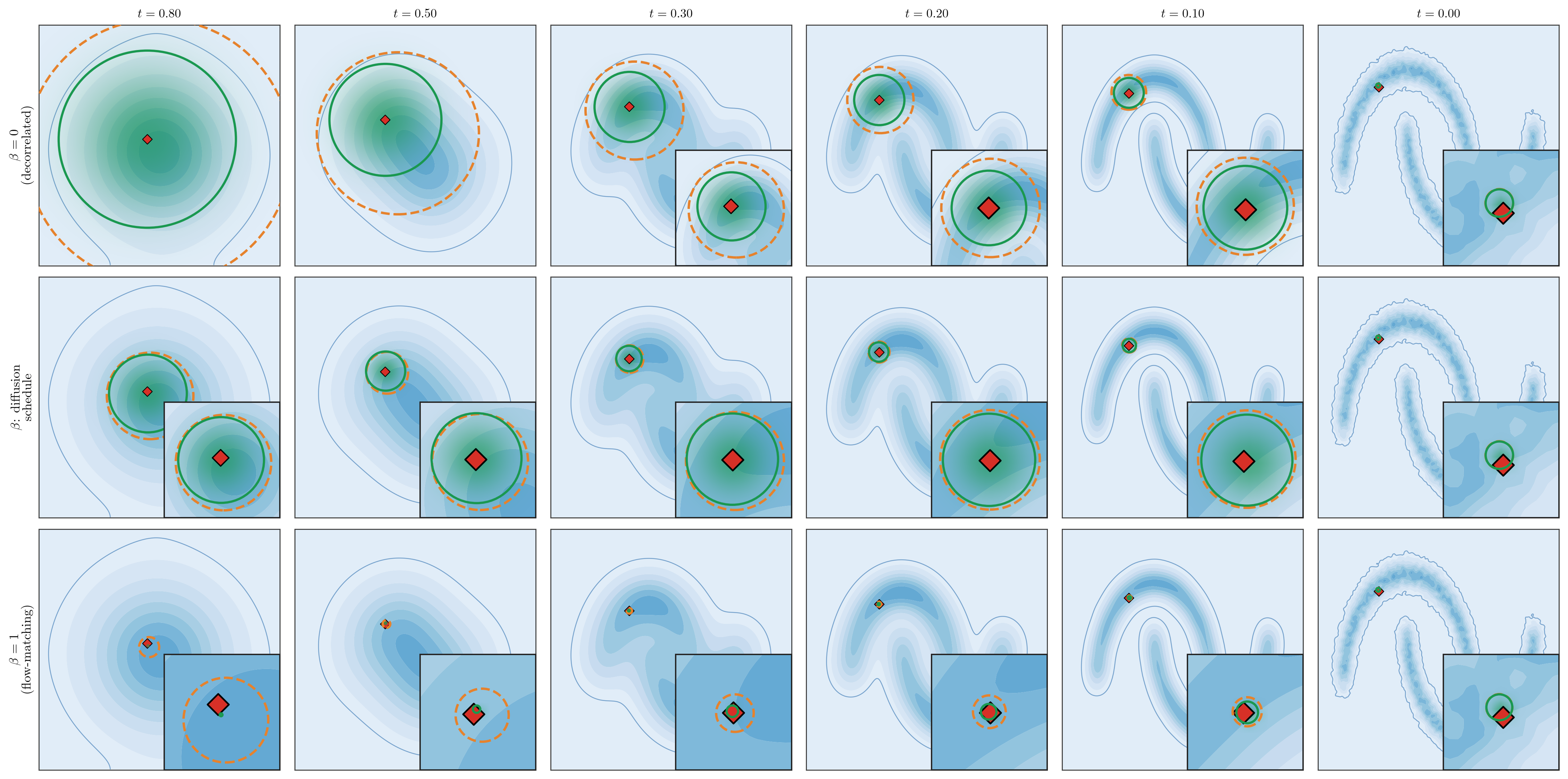}
\caption{The conditioning step at six times along a trajectory (columns)
for three choices of $\beta_t$ (rows): the decorrelated kernel
($\beta_t=0$), the diffusion schedule ($\beta_t=\beta^\star_t$) and the
flow-matching kernel ($\beta_t=1$). Each panel shows the prior over
$x_{t_i}$ before conditioning on $y$ (dashed orange) and the posterior
\eqref{eq:full-update} (solid green), together with the true signal (red
diamond) mapped into current time point (i.e. the red diamonds are showing $\alpha_t y$ for the time associated to that column, as this is what the time conditional posterior is targeting). Larger $\beta_t$ leaves less fresh noise $\tau^2$ in the kernel, so both prior and posterior contract, and at $\beta_t=1$ the conditioning step makes only a small
correction to a nearly deterministic update.}
\label{fig:toy-beta-comparison}
\end{figure}

\section{Implementation Details}
We base our implementation on the github repository for Flower \url{https://github.com/mehrsapo/Flower} which is a fork of the PnP-Flow repository \url{https://github.com/annegnx/PnP-Flow}. We inherit the model architecture, weights, and datasets for the AFHQ-Cats and CelebA from these works. For the FastMRI experiment we trained a model using the same architecture, and for the AAPM model we used the pre-trained model from \cite{denker2026stability}. 

A fully reproducible implementation with pre-trained weights and dataset splits is planned for released in the near future. The following subsections describe the problem-specific implementation details. 

\subsection{AFHQ-Cats and CelebA}
We inherit the experimental setup, pre-trained model weights, datasets/splits, and baseline hyperparameters directly from \cite{pourya2026flower}, which in turn builds on \cite{martin2025pnpflow}.

We evaluate on two datasets. CelebA consists of $128\times128$ human face images; AFHQ-Cat consists of $256\times256$ cat images. All images are normalised to the range $[-1,1]$. Models are trained on the full training splits of both
datasets. 
All metrics are reported on 100 test images per dataset using a consistent split with the comparison methods.

All flow-matching methods (including $\beta$-Flow and the baselines OT-ODE, D-Flow, Flow-Priors, and PnP-Flow) share the same U-Net backbone trained with mini-batch optimal-transport flow matching and a Gaussian latent prior. Pre-trained weights are taken from \cite{martin2025pnpflow}, trained
with learning rate $10^{-4}$: on CelebA for 200 epochs (batch size
128) and on AFHQ-Cat for 400 epochs (batch size 64).

We evaluate on five restoration tasks: \emph{denoising} with additive Gaussian noise ($\sigma_n=0.2$); \emph{Gaussian deblurring} with a $61\times61$ kernel ($\sigma_b=1.0$ for CelebA, $\sigma_b=3.0$ for AFHQ-Cat) and additive noise ($\sigma_n=0.05$); \emph{super-resolution} at $2\times$ for CelebA and $4\times$ for AFHQ-Cat ($\sigma_n=0.05$); \emph{random inpainting} with 70\% of pixels removed ($\sigma_n=0.01$); and \emph{box inpainting} with a centred rectangular mask ($40\times40$ for CelebA, $80\times80$ for AFHQ-Cat, $\sigma_n=0.05$).

We compare against OT-ODE, D-Flow, Flow-Priors, PnP-Flow, PnP-GS, and DiffPIR. Hyperparameters for all baselines
are adopted from \cite{pourya2026flower}, which performed a grid search
on the validation set to maximise PSNR. The number of steps was fixed
to 100 for all methods; the Flower and PnP-Flow results that used
$200^\dag$ or $500^*$ steps on AFHQ-Cat are quoted directly from
\cite{pourya2026flower} and are indicated by the corresponding
superscripts in Table~\ref{tab:afhq}.

For $\beta$-Flow, $\beta_t$ was selected from a uniform grid on $[-0.9, 0.9]$ at spacing $0.1$, augmented with $\pm0.95$ and $\pm0.99$ near the boundaries, and $\kappa$ from $\{0.1,0.5,1,5,10\}$, both tuned independently per dataset and task on the held-out validation split. The resulting values are reported in Table~\ref{tab:hparams_celeba_afhq}.

\begin{table}[h]
\centering
\small
\setlength{\tabcolsep}{4pt}
\caption{$\beta$-Flow hyperparameters ($\beta_t$, $\kappa$) per task,
selected on the validation split to maximise PSNR.}
\label{tab:hparams_celeba_afhq}
\begin{tabular}{@{}lcccc@{}}
\toprule
& \multicolumn{2}{c}{CelebA} & \multicolumn{2}{c}{AFHQ-Cat} \\
\cmidrule(lr){2-3}\cmidrule(lr){4-5}
Task & $\beta_t$ & $\kappa$ & $\beta_t$ & $\kappa$ \\
\midrule
Denoising         &  $0.1$ & $1.0$ &  $0.4$ & $1.0$ \\
Deblurring        & $-0.99$ & $1.0$ & $-0.8$ & $5.0$ \\
Super-resolution  & $-0.9$ & $1.0$ &  $0.5$ & $0.1$ \\
Random inpainting & $-0.9$ & $10.0$ & $-0.9$ & $10.0$ \\
Box inpainting    &  $0.9$ & $0.1$ &  $0.8$ & $0.1$ \\
\bottomrule
\end{tabular}
\end{table}

\subsection{Medical MRI/CT reconstruction}

Every method was tuned on a separate validation set, and the number of steps was fixed to 100 (except for D-Flow). Per-image runtimes for both medical tasks are reported in Table~\ref{tab:mri_ct_runtime}.

For the fastMRI dataset, images are $128\times128$, complex-valued (2 real channels). The forward operator retains $20\%$ of phase-encode lines (yielding an effective $R{=}5$ acceleration), with Gaussian noise $\sigma=0.05$ added directly in k-space. DPS used $\zeta=1.0$, $\Pi$GDM used a guidance scale of 0.5, with the conjugate-gradient solver run for up to 100 iterations (tolerance $10^{-5}$), PnP-Flow used 5 sample averaging and $\rho=1.5$, kept constant throughout, Flower was run with 5 sample averaging (Flower5), D-Flow used an initialization blend of 0.7, a regularization weight of $10^{-8}$, 15 outer LBFGS iterations with 20 inner iterations each, and a 6-step Euler discretization of the flow ODE, $\beta$-Flow used $\beta=-0.9$ and $\kappa=1.0$ and also used 5 sample averaging.

For the CT dataset, images are $256\times256$, single-channel. The forward operator is a sparse-view Radon transform with 60 projection angles over $[0^\circ,180^\circ)$, with Gaussian noise $\sigma=0.01$ on the normalized sinogram. DPS used $\zeta=5.0$, $\Pi$GDM used a guidance scale of 1.5, with the same conjugate-gradient settings as above, PnP-Flow used 5 sample averaging and $\rho=2.0$, kept constant throughout, Flower was again run with 5 sample averaging, D-Flow used an initialization blend of 0.8, with the same remaining settings as on MRI, $\beta$-Flow used $\beta=-0.7$ and $\kappa=5$ and also used 5 sample averaging.

\begin{table}[H]
\centering
\small
\setlength{\tabcolsep}{4pt}
\caption{Wall-clock time per image, in seconds, for the MRI and CT
results of Table~\ref{tab:mri_ct_results}, under the same fixed
step budget of 100 used there (except D-Flow, which uses its own
optimisation budget).}
\label{tab:mri_ct_runtime}
\begin{tabular}{l cc}
\toprule
Method & MRI & CT \\
\midrule
DPS                 & 8.92   & 7.43   \\
$\Pi$GDM            & 9.02   & 21.75  \\
D-Flow              & 312.52 & 283.87 \\
PnP-Flow            & 15.64  & 15.95  \\
Flower          & 23.48  & 48.76  \\
$\beta$-Flow (ours) & 6.67   & 27.10  \\
\bottomrule
\end{tabular}
\end{table}

\section{$\beta_t$ and $\kappa$ Ablation Visualisation}

\begin{figure}[H]
\centering
\includegraphics[width=\linewidth]{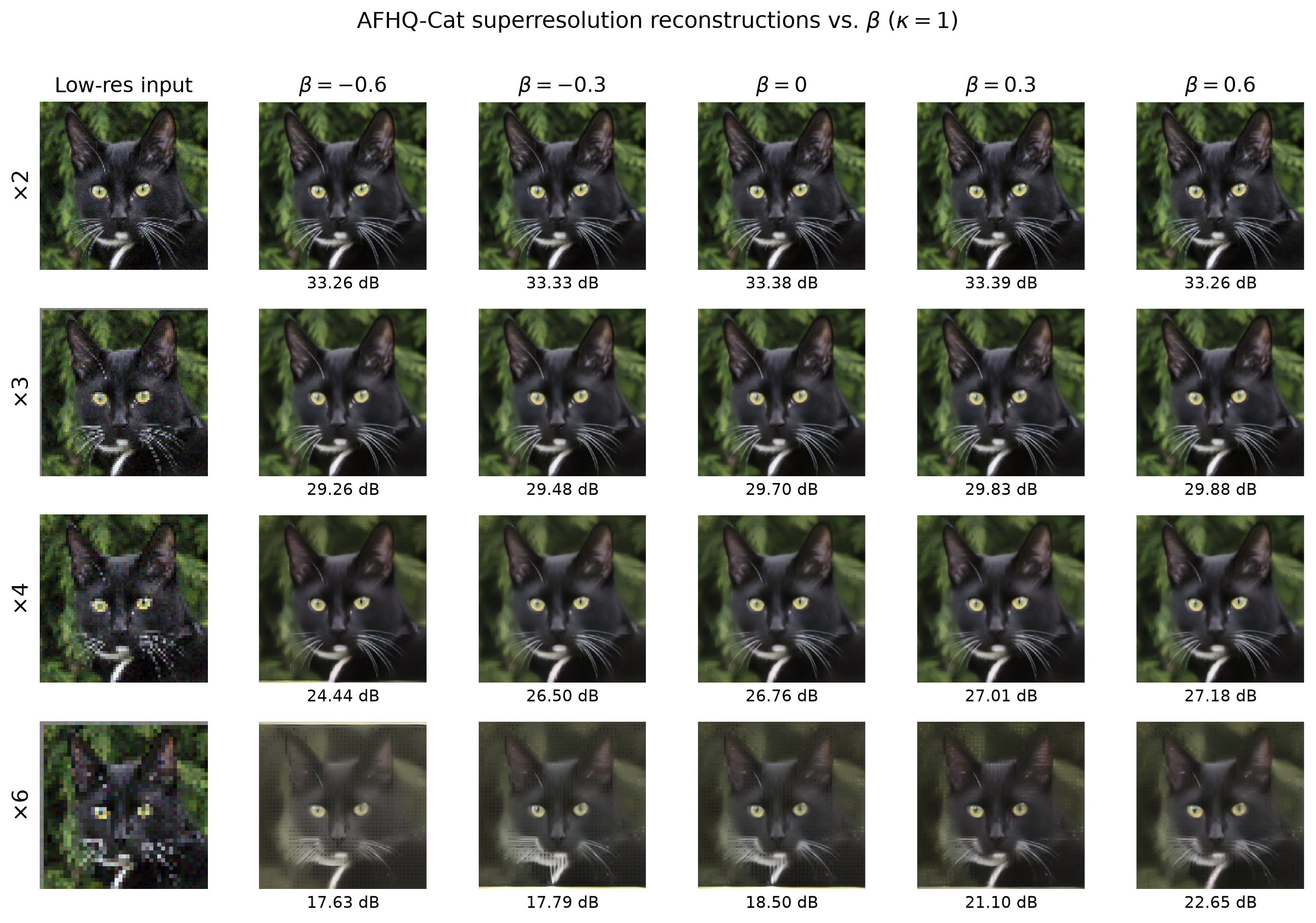}
\caption{Qualitative effect of $\beta_t$ on AFHQ-Cat super-resolution with fixed $\kappa=1$. Leftmost column is the observation, and Columns 2-6 show reconstructions with different $\beta_t=\beta$ parameters. As the downsampling factor increases the observation constrains $x_0$ less, and reconstructions benefit more from increased $\beta_t$ provides. At $\times4$ and $\times6$, $\beta_t=-0.6$ have visible grid artefacts and washed out tone that improve when $\beta_t$ increases, demonstrating the trend in Table~\ref{tab:afhq_sr_beta}.}
\label{fig:afhq_sr_beta_qual}
\end{figure}

\begin{figure}[H]
\centering
\includegraphics[width=\linewidth]{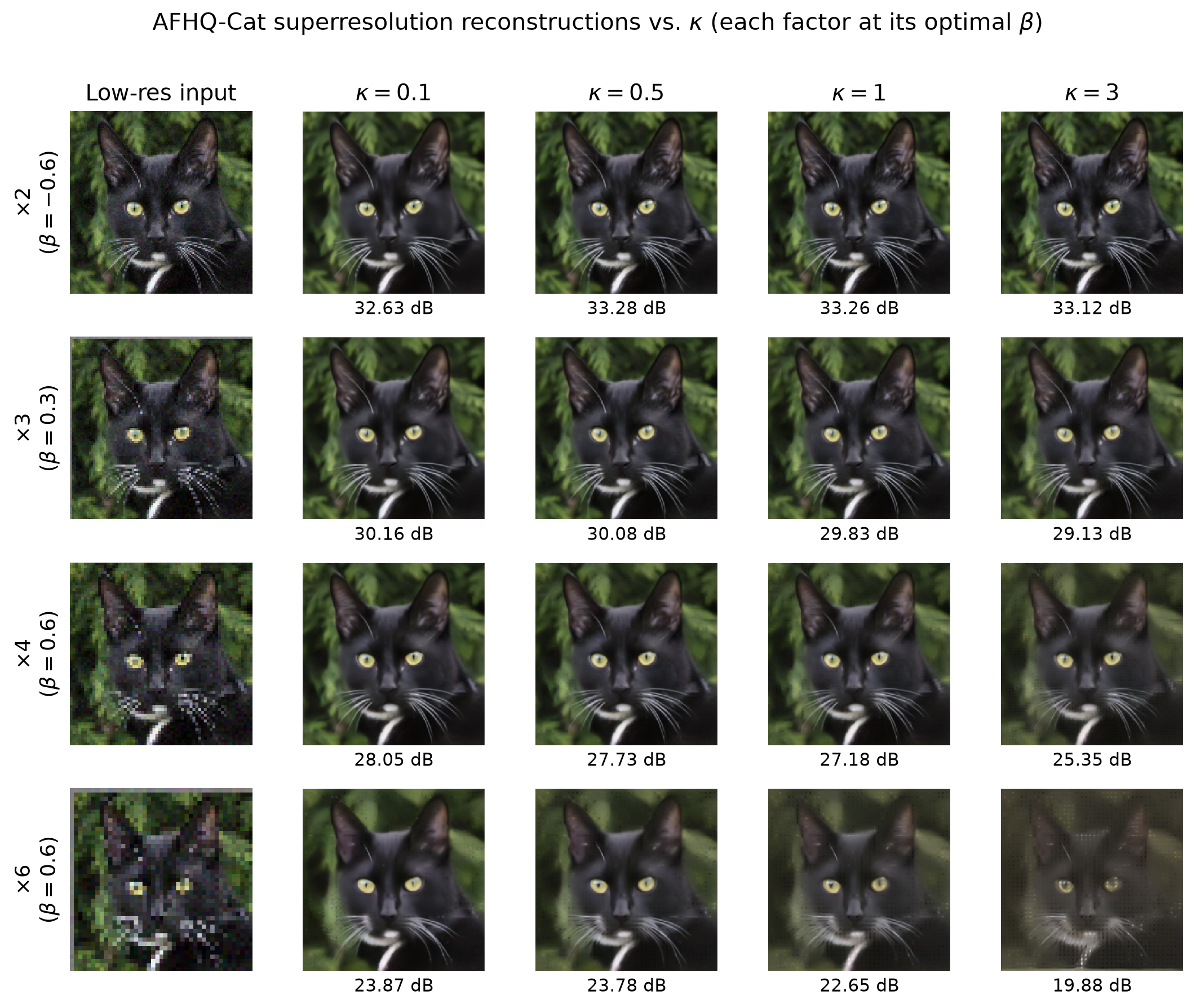}
\caption{Qualitative effect of $\kappa$ on AFHQ-Cat super-resolution using the same test image as Figure~\ref{fig:afhq_sr_beta_qual}. For each scale factor $\beta_t$ is chosen to be its own PSNR-optimal $\beta_t$ ($\times2$: $-0.6$, $\times3$: $0.3$, $\times4,\times6$: $0.6$) from Figure \ref{fig:afhq_sr_beta_qual}. A larger $\kappa$ widens the prior and lets the observation dominate the update; by $\kappa=3$ this produces blurring and a strong grid artefact, most pronounced at $\times4$ and $\times6$ where the observation is least informative, matching the trend in Table~\ref{tab:afhq_sr_kappa}.}
\label{fig:afhq_sr_kappa_qual}
\end{figure}

\end{document}